\documentclass[12pt,english,onecolumn]{IEEEtran}
\usepackage[T1]{fontenc}
\usepackage[latin9]{inputenc}
\usepackage{babel}
\usepackage{float}
\usepackage{amsthm}
\usepackage{amsmath}
\usepackage{amssymb}
\usepackage{wrapfig,subfigure}
\usepackage{graphicx}
\usepackage{times,color}
\usepackage{verbatim}
\usepackage{subfig}
\usepackage{makeidx}

\makeatletter

\floatstyle{ruled}
\newfloat{algorithm}{tbp}{loa}
\providecommand{\algorithmname}{Algorithm}
\floatname{algorithm}{\protect\algorithmname}

\theoremstyle{plain}
\newtheorem{thm}{\protect\theoremname}
\theoremstyle{definition}
\newtheorem{defn}[thm]{\protect\definitionname}
\theoremstyle{plain}
\newtheorem{cor}[thm]{\protect\corollaryname}
\theoremstyle{remark}
\newtheorem{rem}[thm]{\protect\remarkname}
\theoremstyle{lemma}
\newtheorem{lem}[thm]{\protect\lemmaname}

\makeatother

\providecommand{\corollaryname}{Corollary}
\providecommand{\definitionname}{Definition}
\providecommand{\lemmaname}{Lemma}
\providecommand{\remarkname}{Remark}
\providecommand{\theoremname}{Theorem}

\global\long\def\mc#1{\mathcal{#1}}

\global\long\def\mca{\mathcal{A}}
\global\long\def\mco{\mathcal{O}}
\global\long\def\mcs{\mathcal{S}}

\begin{document}

\title{Robust High Dimensional Sparse Regression and Matching Pursuit}
\author{{
Yudong Chen, Constantine Caramanis and Shie Mannor}\footnote{ydchen@utexas.edu, constantine@utexas.edu, shie@ee.technion.ac.il}}

\maketitle

\newcommand{\fix}{\marginpar{FIX}}
\newcommand{\new}{\marginpar{NEW}}

\begin{abstract}
In this paper we consider high dimensional sparse regression, and develop strategies able to deal with arbitrary -- possibly, severe or coordinated -- errors in the covariance matrix $X$. These may come from corrupted data, persistent experimental errors, or malicious respondents in surveys/recommender systems, etc. Such non-stochastic error-in-variables problems are notoriously difficult to treat, and as we demonstrate, the problem is particularly pronounced in high-dimensional settings where the primary goal is {\em support recovery} of the sparse regressor. We develop algorithms for support recovery in sparse regression, when some number $n_1$ out of $n+n_1$ total covariate/response pairs are {\it arbitrarily (possibly maliciously) corrupted}. We are interested in understanding how many outliers, $n_1$, we can tolerate, while identifying the correct support. To the best of our knowledge, neither standard outlier rejection techniques, nor recently developed robust regression algorithms (that focus only on corrupted response variables), nor recent algorithms for dealing with stochastic noise or erasures, can provide guarantees on support recovery. Perhaps surprisingly, we also show that the natural brute force algorithm that searches over all subsets of $n$ covariate/response pairs, and all subsets of possible support coordinates in order to minimize regression error, is remarkably poor, unable to correctly identify the support with even $n_1 = O(n/k)$ corrupted points, where $k$ is the sparsity. This is true even in the basic setting we consider, where all authentic measurements and noise are independent and sub-Gaussian. In this setting, we provide a simple algorithm -- no more computationally taxing than OMP -- that gives stronger performance guarantees, recovering the support with up to $n_1 = O(n/(\sqrt{k} \log p))$ corrupted points, where $p$ is the dimension of the signal to be recovered.
\end{abstract}

\section{Introduction}

Linear regression and sparse linear regression seek to express a response variable as the linear combination of (a small number of) covariates. They form one of the most basic procedures in statistics, engineering, and science. More recently, regression has found increasing applications in the high-dimensional regime, where the number of variables, $p$, is much larger than the number of measurements or observations, $n$. Applications in biology, genetics, as well as in social networks, human behavior prediction and recommendation, abound, to name just a few. The key structural property exploited in high-dimensional regression, is that the regressor is often sparse, or near sparse, and as much recent research has demonstrated, in many cases it can be efficiently recovered, despite the grossly underdetermined nature of the problem (e.g., \cite{chen1999basispursuit,candes2005decoding,candes2007dantzig,davenport2010OMPRIP,wainwright2009sharp}). Another common theme in large-scale learning problems -- particularly problems in the high-dimensional regime -- is that we not only have big data, but we have dirty data. Recently, attention has focused on the setting where the output (or response) variable and the matrix of covariates are plagued by erasures, and/or by stochastic additive noise \cite{loh2012nonconvex,rosenbaum2010sparse,rosenbaum2011improved,chen2012eiv,chen2013eivICML}. Yet many applications, including those mentioned, may suffer from persistent errors, that are ill-modeled by stochastic distribution; indeed, many applications, particularly those modeling human behavior, may exhibit maliciously corrupted data. 

This paper is about extending the power of regression, and in particular, sparse high-dimensional regression, to be robust to this type of noise. We call this {\em deterministic} or {\em cardinality constrained} robustness, because rather than restricting the magnitude of the noise, or any other such property of the noise, we merely assume there is a bound on how many data points, or how many coordinates of every single covariate, are corrupted. Other than this number, we make absolutely no assumptions on what the adversary can do -- the adversary is virtually unlimited in computational power and knowledge about our algorithm and about the authentic points. There are two basic models we consider. In both, we assume there is an underlying generative model: $y = X \beta^{\ast} + e$, where $X$ is the matrix of covariates, $e$ is sub-Gaussian noise. In the {\em row-corruption model}, we assume that each pair of covariates and response we see is either due to the generative model, i.e., $(y_i,X_i)$, or is corrupted in some arbitrary way, with the only restriction that at most $n_1$ such pairs are corrupted. In the {\em distributed corruption model}, we assume that $y$ and {\em each column} of $X$, has $n_1$ elements that are arbitrarily corrupted (evidently, the second model is a strictly harsher corruption model). Building efficient regression algorithms that recover at least the support of $\beta^{\ast}$ accurately subject to even such deterministic data corruption, greatly expands the scope of problems where regression can be productively applied. The basic question is when is this possible -- how big can $n_1$ be, while still allowing correct recovery of the support of $\beta^{\ast}$.

Many sparse-regression algorithms have been proposed, and their properties under clean observations are well understood; we survey some of these results in the next two sections. Also well-known, is that the performance of standard algorithms (e.g., Lasso, Orthogonal Matching Pursuit) breaks down even in the face of just a few corrupted points or covariate coefficients. As more work has focused on robustness in the high-dimensional regime, it has also become clear that the techniques of classical robust statistics such as outlier removal preprocessing steps cannot be applied to the high-dimensional regime \cite{donohoharvard, Huber1981Robstat}. The reason for this lies in the high dimensionality. In this setting, identifying outliers {\it a priori} is typically impossible: outliers might not exhibit any strangeness in the ambient space due to the high-dimensional noise (see \cite{xu2013hrpca} for a further detailed discussion), and thus can be identified only when the true  low-dimensional structure is (at least approximately) known; on the other hand, the true structure cannot be computed by ignoring outliers. Other classical approaches have involved replacing the standard mean squared loss with a trimmed variant or even median squared loss \cite{hampel2011robust}; first, these are non convex, and second, it is not clear that they provide any performance guarantees, especially in high dimensions.

Recently, the works in \cite{laska2009jp,nguyen2011densecorrupt,wright2010dense,xiaodongli} have proposed an approach to handle arbitrary corruption in the \emph{response variable}. As we show, this approach faces serious difficulties when the \emph{covariates} is also corrupted, and is bound to fail in this setting. One might modify this approach in the spirit of Total Least Squares (TLS) \cite{zhu2011tls} to account for noise in the covariates (discussed in Section \ref{sec:relatedwork}), but it leads to highly non convex problems. Moreover, the approaches proposed in these papers are the natural convexification of the (exponential time) brute force algorithm that searches over all subsets of covariate/response pairs (i.e., rows of the measurement matrix and corresponding entries of the response vector) and subsets of the support (i.e., columns of the measurement matrix) and then returns the vector that minimizes the regression error over the best selection of such subsets. Perhaps surprisingly, we show that the brute force algorithm itself has remarkably weak performance. Another line of work has developed approaches to handle stochastic noise or small bounded noise in the covariates \cite{herman2010perturbation, rosenbaum2010sparse,rosenbaum2011improved,loh2012nonconvex,chen2012eiv}. The corruption models in there, however, are different from ours which allows arbitrary and malicious noise; those results seem to depend crucially on the assumed structure of the noise and cannot handle the setting in this paper. 

More generally, even beyond regression, in, e.g., robust PCA and robust matrix completion \cite{chandrasekaran2011siam, candes2009robustPCA, Xu2010RobustPCA, Chen2010RobustCF, lerman2012robustpcaneedle}, recent robust recovery in high dimensions results have for the most part depended on convex optimization formulations. We show in Section \ref{sec:failure} that for our setting, {\it convex-optimization based approaches that try to relax the brute-force formulation fail to recover support, with even a constant number of outliers}. Accordingly, we develop a different line of robust algorithms, focusing on Greedy-type approaches like Matching Pursuit (MP).

In summary, to the best of our knowledge, no robust sparse regression algorithm has been proposed that can provide performance guarantees, and in particular, guaranteed support recovery, under {\it arbitrarily and maliciously corrupted} covariates and response variables.

We believe robustness is of great interest both in practice and in theory. Modern applications often involve ``big but dirty data'', where outliers are ubiquitous either due to adversarial manipulation or to the fact some samples are generated from a model different from the assumed one. It is thus desirable to develop  robust sparse regression procedures. From a theoretical perspective, it is somewhat surprising that the addition of a few outliers can transform a simple problem to a hard one; we discuss the difficulties in more detail in the subsequent sections.

{\bf Paper Contributions:}
In this paper, we propose and discuss a simple (in particular, efficient) algorithm for robust sparse regression, for the setting where both covariates and response variables are arbitrarily corrupted, and show that our algorithm guarantees support recovery under far more corrupted covariate/response pairs than any other algorithm we are aware of. We briefly summarize our contributions here:
\begin{enumerate}
\item We consider the corruption model where $n_1$ rows of the covariate matrix $X$ and the response vector $y$ are arbitrarily corrupted. We demonstrate that other algorithms we are aware of, including standard convex optimization approaches and the natural brute force algorithm, have very weak (if any) guarantees for support recovery.
\item For the corruption model above, we give support recovery guarantees for our algorithm, showing that we correctly recover the support with $n_1 = O(n/(\sqrt{k} \log p))$ arbitrarily corrupted response/covariate pairs.
\item We consider a stronger corruption model, where instead of $n_1$ corrupted rows of the matrix $ X $, {\it each column can have up to $n_1$ arbitrarily corrupted entries}. We show that our algorithm also works in this setting, with precisely the same recovery guarantees. To the best of our knowledge, this problem has not been previously considered.
\end{enumerate}

\section{Problem Setup}

We consider the problem of sparse linear regression. The unknown parameter
$\beta^{*}\in\mathbb{R}^{p}$ is assumed to be $k$-sparse ($k<p$),
i.e., has only $k$ nonzeros. The observations take the form of covariate-response
pairs $(x_{i},y_{i})\in\mathbb{R}^{p}\times\mathbb{R}$, $i=1,\ldots,n+n_{1}$.
Among these, $n$ pairs are authentic samples obeying the following linear model
\[
y_{i}=\left\langle x_{i},\beta^{*}\right\rangle +e_{i},
\]
where $e_{i}$ is additive noise, and $p \ge n$. For corruption, we consider the following two models.

\begin{defn}[Row Corruption]
The $n_{1}$ pairs are arbitrarily corrupted, with \emph{both} $x_{i}$
and $y_{i}$ being potentially corrupted.
\end{defn} 

\begin{defn}[Distributed Corruption]
We allow arbitrary corruption of \emph{any} $n_1/2$ elements of each column of the covariate matrix $X$ and of the response $y$. 
\end{defn} 
In particular, the corrupted entries need not lie in the same $n_1$ rows. Clearly this includes the previous model as a special case up to a constant factor of $2$.

Note that in both models, we impose \emph{no assumption whatsoever on the corrupted pairs}. They might be unbounded, non-stochastic, and even dependent on the authentic samples. They are unconstrained other than in their cardinality -- the number of rows or coefficients corrupted. We illustrate both of these corruption models pictorially in Figure \ref{fig:robReg}.

\textbf{Goal:} Given these observations $\{(x_i,y_i)\}$, the goal is to obtain a reliable estimate $\hat{\beta}$ of $\beta^{*}$ with correct support and bounded error $\left\Vert \hat{\beta}-\beta^{*}\right\Vert _{2}$. A fundamental question, therefore, is to understand in each given model, given $p,n,$ and $k$, how many outliers $(n_{1})$ an estimator can handle. 

\begin{rem} 
This is a strong notion of robustness. In particular, requiring support recovery is a more stringent requirement than requiring, for example, bounded distance to the true solution, or bounded loss degradation. It is also worth noting that robustness here is a completely different notion than robustness in the sense of Robust Optimization (e.g., \cite{BertsimasBrownCaramanis2011}). There, we seek a solution that minimizes the error we incur when an adversary perturbs our loss function. In contrast, here, here is a true generative model that obeys the structural assumptions of the problem (namely, sparsity of $\beta^{\ast}$), but an adversary corrupts the data that ordinarily provide us with the true input-output behavior of that model.
\end{rem}

We emphasize that our setting is fundamentally different from those that only allow corruptions in $y$, and robust sparse regression techniques that only consider corruption in $ y $ are bound to fail in our setting. We elaborate on this point in Section \ref{sec:failure} and \ref{sec:bruteforce}. Moreover, under the distributed corruption model, there is no hope of considering an equivalent model with only corruptions in $ y $: all entries of $ y $ could be corrupted -- an absurd setting to hope for a solution.

\begin{figure}[h]
\begin{center}
\begin{tabular}{cc}
\includegraphics[width=.4\linewidth]{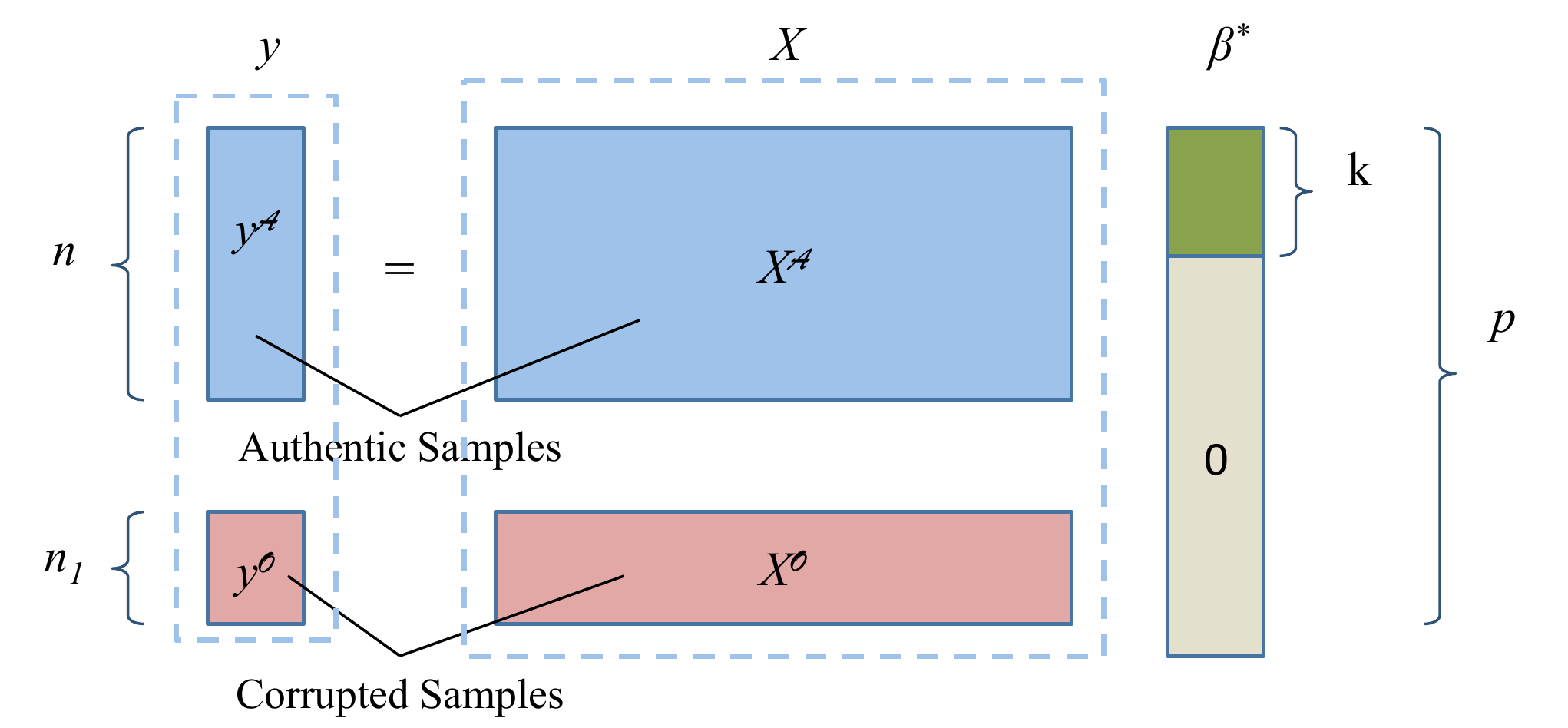} &
\includegraphics[width=.4\linewidth]{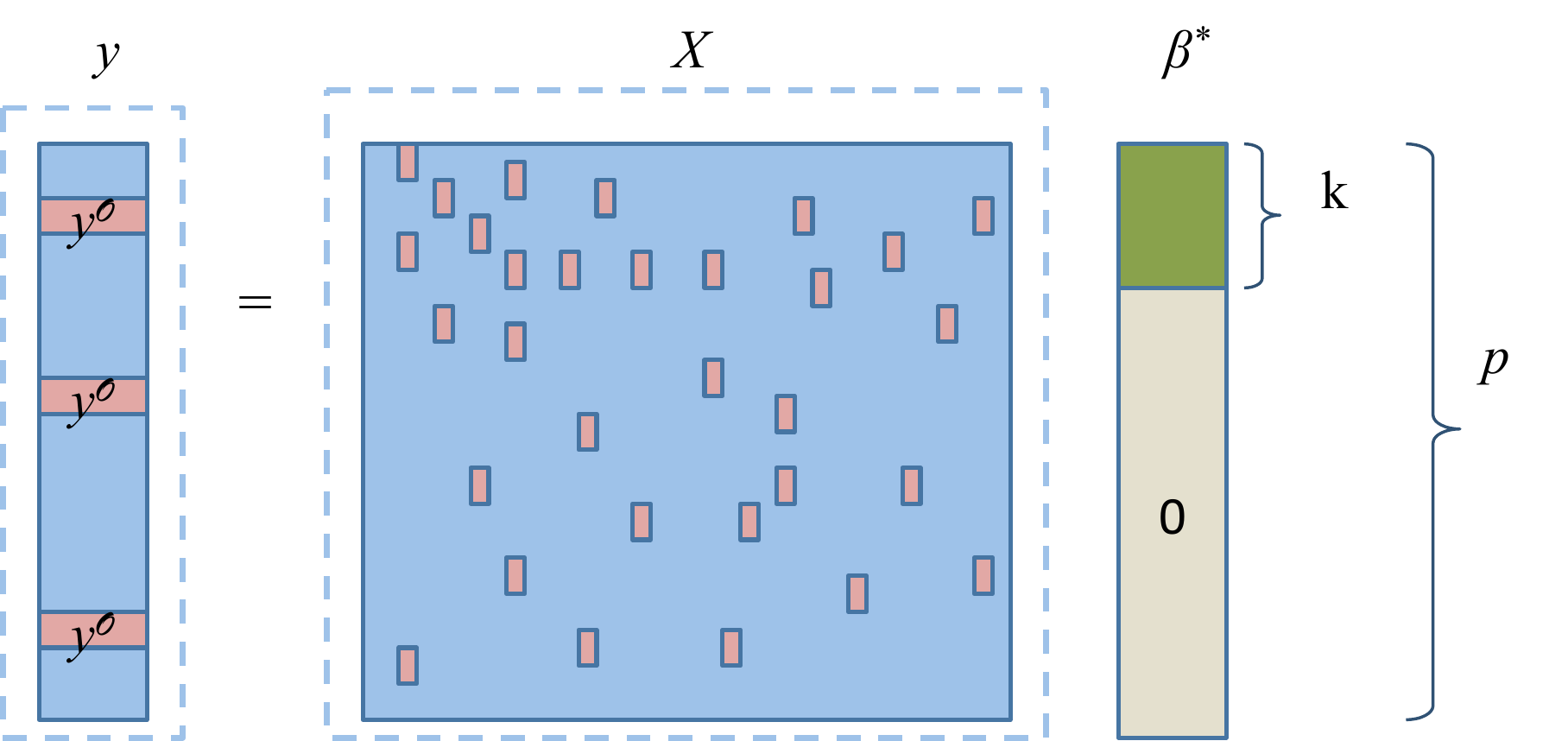} \\ (a) & (b)
%\widgraph{.4\textwidth}{Model1Error.pdf} &
%\widgraph{.4\textwidth}{Model2Error.pdf} \\ (a) & (b)
\end{tabular}
\end{center}
\caption{Consider identifying the dominant genetic
      markers for a disease; entire experiments might fail, or
      different gene expression levels might be erroneously read
      across different experiments. Panel (a) corresponds to the first
      type of error, where $n_1$ samples $(y_i,X_i)$ are arbitrarily
      corrupted. This is our first error model. Panel (b) illustrates an instance of the second
      scenario, and thus our second error model: in each experiment, several covariates may be misread,
      but no single covariate is misread more than $n_1$ times. }
\label{fig:robReg}
\end{figure}

\section{Related Work}
\label{sec:relatedwork}

Under the high-dimensional setting $p\ge n$, there is a large body of literature on sparse recovery when there is no corruption. It is now well-known that recovery of $\beta^{*}$
is possible only when the covariate matrix $ X $
satisfies certain conditions, such as the Restricted Isometry/Eigenvalue Property \cite{candes2005decoding,bickel2009simultaneous}, Mutual Incoherence Condition \cite{donoho2006stable} or Exact Recovery Condition \cite{tropp2004greed}. Various ensembles of random
matrices are shown to satisfy these conditions with high probability.
Many estimators have been proposed, most notably Basis Pursuit (a.k.a.
Lasso) \cite{tibshirani1996lasso,donoho2006stable},
which solves an $\ell_{1}$-regularized least squares problem
\begin{eqnarray*}
\min_{\beta} & & \left\Vert y - X\beta \right\Vert _{2}+\lambda\left\Vert \beta\right\Vert _{1},
\end{eqnarray*}
as well as Orthogonal Matching Pursuit (OMP) \cite{donoho2006stable, tropp2004greed},
which is a greedy algorithm that estimates the support of $\beta^{*}$
sequentially. Both Lasso and OMP, as well as many other estimators,
are guaranteed to recover $\beta^{*}$ with good accuracy when
$X$ is well-conditioned, and the number of observations satisfies $n\gtrsim k\log p.$ (Here we mean there exists a constant $ c $, independent of $ k,n,p $, such that the statement holds. We use this notation throughout the paper.) Moreover, this condition is also shown to be  necessary; see, e.g., \cite{wainwright2009lassoLimit}.

Most existing methods are not robust to outliers; for example, standard
Lasso and OMP fail even if only one entry of $ X $ or $y$ is corrupted. One might consider a natural modification of Lasso in the spirit of Total Least Squares, and solve 
\begin{eqnarray}
\label{eq:TLS}
\min_{\beta, E} \Vert (X-E)\beta - y \Vert_2 + \lambda \Vert \beta \Vert_1 + \eta \Vert E \Vert_*,
\end{eqnarray}
where $ E $ accounts for corruption in the covariate matrix, and $\Vert \cdot \Vert_*$ is a norm. When $ E $ is known to be row sparse (as is the case in our row-corruption model), one might choose $ \Vert \cdot \Vert_*$ to be $\Vert \cdot \Vert_{1,2} $ or $\Vert \cdot \Vert_{1,\infty}$\footnote{$\Vert E \Vert_{1,2} $ ($\Vert E \Vert_{1,\infty}$) is the sum of the $\ell^2$ ($\ell^\infty$, respectively) norms of the rows of $ E $.}; the work in \cite{zhu2011tls} considers using $ \Vert \cdot \Vert_* =\Vert \cdot \Vert_F $ (similar to TLS), which is more suitable when $ E $ is dense yet bounded. The optimization problem \eqref{eq:TLS} is, however, highly non convex due to the bilinear term $ E\beta $, and no tractable algorithm with provable performance guarantees is known. 

Another modification of Lasso accounts for the corruption in the response via an additional variable $ z$ \cite{laska2009jp,nguyen2011densecorrupt,wright2010dense,xiaodongli}:
\begin{eqnarray}
\min_{\beta,z} \; \left\Vert X\beta-y-z\right\Vert _{2}+\lambda\left\Vert \beta\right\Vert _{1}+\gamma\left\Vert z\right\Vert _{1}.
\label{eq:xiaodong_noisy}
\end{eqnarray}
We call this approach Justice Pursuit (JP) after \cite{laska2009jp}. Unlike the previous approach, the problem \eqref{eq:xiaodong_noisy} is convex. In fact, it is the natural convexification of the brute force algorithm:
\begin{eqnarray}
\min_{\beta,z} && \left\Vert X\beta-y-z\right\Vert _{2} \label{eq:bruteforece1}\\
{\rm s.t.}: && \|\beta\|_0 \leq k \nonumber \\
&& \|z\|_0 \leq n_1, \nonumber
\end{eqnarray}
where $ \Vert u \Vert_0 $ denotes the number of nonzero entries in $ u $. It is easy to see (and well known) that the so-called Justice Pursuit relaxation \eqref{eq:xiaodong_noisy} is equivalent to minimizing the Huber loss function plus the $\ell_1$ regularizer, with an explicit relation between $ \gamma$ and the parameter of the Huber loss function \cite{fuchs1999robustregression}.  Formulation \eqref{eq:xiaodong_noisy} has excellent recovery guarantees when \emph{only} the response variable is corrupted, delivering exact recovery under a constant fraction of outliers. However, we show in the next section that a broad class of convex optimization-based approaches, with \eqref{eq:xiaodong_noisy} as a special case, fail when the covariate $X$ is also corrupted. In the subsequent section, we show that even the original brute force formulation is problematic: while it can recover from some number $n_1$ of corrupted rows, that number is order-wise worse than what the algorithm we give can guarantee.

We also note that neither the brute force algorithm above, nor its relaxation, JP, are appropriate for our second model for corruption. Indeed, in this setting, modeling via JP would require handling the setting where {\em every single entry} of the output variable, $y$, is corrupted, something which certainly cannot be done.

For standard linear regression problems in the
classical scaling $n\gg p$, various robust estimators have been proposed,
including $M$-, $R$-, and $S$-estimators \cite{Huber1981Robstat,maronna2006robuststat},
as well as those based on $\ell_{1}$-minimization \cite{kekatos2011robustsensing}.
Many of these estimators lead to non-convex optimization problems,
and even for those that are convex, it is unclear how they can be
used in the high-dimensional scaling with sparse $\beta^{*}$. Another
difficulty in applying classical robust methods to our problems arises
from the fact that the covariates, $x_{i}$, also lie in a high-dimensional
space, and thus defeat many outlier detection/rejection techniques that might
otherwise work well in low-dimensions. Again, for our second model of corruption, outlier detection seems even more hopeless.

\section{Failure of the Convex Optimization Approach}\label{sec:failure}

We consider a broad class of convex optimization-based approaches of the following form:
\begin{eqnarray}
\min_\beta & & f(y-X\beta)  \\
\textrm{s.t.} & & h(\beta) \le R. \nonumber
\end{eqnarray}
Here $R$ is a radius parameter that can be tuned. Both $f(\cdot)$ and $h(\cdot)$ are convex functions, which can be interpreted as a loss function (of the residual) and a regularizer (of $\beta$), respectively. For example, one may take $f(v)=\min_z \|v-z\|_2 + \gamma \|z\|_1$ and $h(\beta)=\|\beta\|_1$, which recovers the Justice Pursuit \eqref{eq:xiaodong_noisy} by Lagrangian duality; note that this $f(v)$ is convex because 
\begin{eqnarray*}
\|\alpha v_1 + (1-\alpha) v_2 - z \|_2 + \gamma \|z\|_1 
\le \alpha (\|v_1-z\|_2 + \gamma \|z\|_1) + (1-\alpha) (\|v_1-z\|_2  + \gamma \|z\|_1),
\end{eqnarray*}
by sub-additivity of norms. The function $f(\cdot)$ can also be any other robust convex loss function including the Huber loss function.

We assume that $f(\cdot)$ and $h(\cdot)$ obey a very mild condition, which is satisfied by any non-trivial loss function and regularizer that we know of. In the sequel we use $ [z_1; z_2] $ to denote the concatenation of two column vectors $ z_1 $ and $ z_2 $. 
\begin{defn}[Standard Convex Optimization (SCO) Condition]
We say $f(\cdot)$ and $h(\cdot)$ satisfy the SCO Condition if $ \lim_{\alpha \rightarrow \infty}f( \alpha v ) =\infty $ for all $v\neq 0$, $f([v_1;v_2]) \ge f([0;v_2])$ for all $v_1, v_2$, and $ h(\cdot) $ is invariant under permutation of coordinates. 
\end{defn}

We also assume $R\ge h(\beta^*)$ because otherwise the formulation is not consistent even when there are no outliers. The following theorem shows that under this assumption, the convex optimization approach fails when both $X$ and $y$ are corrupted. We only show this for our first corruption model, since it is a special case of the second distributed model. As illustrated in Figure \ref{fig:robReg}, let $\mathcal{A}$ and $\mathcal{O}$ be the (unknown) sets of indices corresponding to authentic and corrupted observations, respectively, and $X^{\mathcal{A}}$ and $X^{\mathcal{O}}$ be the authentic and corrupted rows of the covariate matrix $X=\left[x_{1}, \dots, x_{n+n_1}\right]^{\top}$. The vectors $y^{\mathcal{A}}$ and $y^{\mathcal{O}}$ are defined similarly. Also let $\Lambda^*$ be the support of $\beta^*$. With this notation, we have the following.

\begin{thm}
Suppose $f$ and $h$ satisfy the SCO Condition. When $n_1\ge 1$ and $k\ge 1$, the adversary can corrupt $X$ and $y$ in such a way that  for all $R$ with $R\ge h(\beta^*)$ the optimal solution does not have the correct support.
\end{thm}
\begin{proof}
Recall that $y$ = $[y^\mca; y^\mco]$ and $X$ = $[X^\mca; X^\mco]$ with $y^\mca = X^\mca \beta^* +e$, and $ \Lambda^* $ is the true support. The adversary fixes some set $\hat{\Lambda}$ disjoint from the true support $\Lambda^*$ with $\vert\hat{\Lambda}\vert = \vert\Lambda^*\vert$. It then chooses $\hat{\beta}$ and $y^\mco$ such that $\hat{\beta}_{\hat{\Lambda}} = \beta^*_{\Lambda^*}$ $\hat{\beta}_{\hat{\Lambda}^c} =0$,  and $y^\mco = X^\mco \hat{\beta}$ with $X^\mco$ to be determined later. By assumption we have $h(\hat{\beta})=h(\beta^*)\le R$, so $\hat{\beta}$ is feasible. Its objective value is $f(y-X\hat{\beta}) = f([y^\mca-X^\mca_{\hat{\Lambda}} \beta^*_{\Lambda^*}; 0])\le C$ for some finite constant $ C $. The adversary further chooses $X^\mco$ such that $X^\mco_{\Lambda^*} = 0$ and $X^\mco_{\Lambda}$ is large. Any  $\tilde{\beta}$ supported on $\Lambda^*$ has objective value 
\begin{eqnarray*}
f(y-X\tilde{\beta}) =  f([y^\mca-X^\mca \tilde{\beta}; X^\mco (\hat{\beta} -\tilde{\beta})]) =  f([y^\mca-X^\mca \tilde{\beta}; X^\mco_{\Lambda} \beta^*_{\Lambda
^*}]) \ge f([0; X^\mco_{\Lambda} \beta^*_{\Lambda
^*}]),
\end{eqnarray*} 
which can be made bigger than $ C $ under the SCO Condition. Therefore, any solution $\tilde{\beta}$ with the correct support $\Lambda^*$ has a higher objective value than $\hat{\beta}$, and thus is not the optimal solution.
\end{proof}

Our proof proceeds by using a simple corruption strategy. Certainly, there are natural approaches to deal with this specific example, e.g., removing entries of $X$ with large values. But discarding such large-value entries is not enough, as there may exist more sophisticated corruption schemes where simple magnitude-based clipping is ineffective. We illustrate this with a concrete example in the simulation section, where Justice Pursuit along with large-value-trimming fails to recover the correct support. Indeed, this example serves merely to illustrate more generally the inadequacy of a purely convex-optimization-based approach.

More importantly, while the idea of considering an unbounded outlier is not new and has been used in classical Robust Statistics and more recently in \cite{yu2012polynomialRobust}, the above theorem highlights the sharp contrast between the success of convex optimization (e.g., JP) under corruption in only $ y $, and its complete failure when both $ X $ and $y$ are corrupted. Corruptions in $X$ not only break the linear relationship between $y$ and $X$, but also destroy properties of $X$ necessary for existing sparse regression approaches. In the high dimensional setting where support recovery is concerned, there is a fundamental difference between the hardness of the two corruption models.

\section{The Natural Brute Force Algorithm}
\label{sec:bruteforce}

The brute force algorithm \eqref{eq:bruteforece1} can be restated as follows: it looks at all possible $n\times k$
submatrices of $X$ and picks the one that gives the smallest regression
error w.r.t. the corresponding subvector of $y$. Formally, let $X^\mcs_{\Lambda}$
denote the submatrix of $X$ corresponding to row indices $\mc S$
and column indices $\Lambda$, and let $y^{\mcs}$ denote the subvector
of $y$ corresponding to indices $\mcs$. The algorithm solves
\begin{eqnarray}
\min_{\theta \in \mathbb{R}^k, \mcs,\Lambda} & & \left\Vert y^{\mcs}-X^{\mcs}_{\Lambda}\theta\right\Vert _{2}\label{eq:optalg1}\\
s.t. & & \left|\mc S\right|=n,\nonumber \\
& & \left|\Lambda\right|=k. \nonumber
\end{eqnarray}
Suppose the optimal solution is $\hat{\mc S},\hat{\Lambda},\hat{\theta}$.
Then, the algorithm outputs $\hat{\beta}$ with $\hat{\beta}_{\Lambda}=\theta$
and $\hat{\beta}_{\Lambda^{c}}=0$. Note that this algorithm has exponential
complexity in $ n $ and $ k $, and $\mathcal{S}^c$ can be considered as an operational definition of
outliers. We show that even this algorithm has poor performance and cannot handle large $n_{1}$.

To this end, we consider the simple Gaussian design model, where the entries of $ X^\mca $ and $ e $ are independent zero-mean Gaussian random variables with variance $ \frac{1}{n} $ and $ \frac{\sigma_e}{n} $, respectively. The $\frac{1}{n}$ factor is simply for normalization and no generality
is lost. We consider the setting where $\sigma_e^2 = k$ and $\beta_{\Lambda^{*}}^{*}=[1,\ldots,1]^{\top}$. If $ n_1=0 $, existing methods (e.g., Lasso and standard OMP), and the brute force algorithm as well, can recover the support of $ \beta^* $  {with high probability} provided $n\gtrsim k\log p$. Here and henceforth, by \emph{with high probability (w.h.p.)} we mean with probability at least $1-p^{-2}$. However, when there are outliers, we have the following negative result.
\begin{thm}
\label{thm:bruteforce}
Under the above setting, if $n\gtrsim k^{3}\log p$ and $n_{1}\gtrsim\frac{3n}{k+1},$ then
the adversary can corrupt $X$ and $y$ in such a way that the brute force
algorithm does not output the correct support $\Lambda^{*}$.
\end{thm}
The proof is given in Section \ref{sec:proofs}. We believe the condition $n\gtrsim k^{3}\log p$ is an artifact of our proof and is not necessary. This theorem shows that the brute force algorithm can only handle $O\left(\frac{n}{k}\right)$ outliers. In the next section, we propose a simple, tractable algorithm that {\it outperforms} this brute force algorithm and can handle $O\left(\frac{n}{\sqrt{k}}\right)$ outliers.

\section{Proposed Approach: Robust Matching Pursuit}

The discussion in the last two sections demonstrates that standard techniques
for high-dimensional statistics and robust statistics are inadequate to handle our problem. As mentioned in the introduction, we believe the
key to obtaining an effective robust estimator for high-dimensional data, is simultaneous structure identification and
outlier rejection. In particular, for the sparse recovery problem
where the observations $Z_{i}=(x_{i},y_{i})$ reside in a high-dimensional
space, it is crucial to utilize the low-dimensional structure of $\beta^{*}$
and perform outlier rejection in the {}``right'' low-dimensional space
in which $\beta^{*}$ lies. In this section, we propose a candidate
algorithm, called Robust Matching Pursuit (RoMP), which is based
on this intuition.

Standard MP estimates the support of $\beta^{*}$ sequentially.
At each step, it selects the column of $X$ which has the largest
(in absolute value) inner product with the current residual $ r $, and adds
this column to the set of previously selected columns. The algorithm iterates until some stopping criterion is met. If the sparsity level $k$ of $\beta^{*}$
is known, then one may stop MP after $k$ iterations.

To successfully recover the support of $ \beta^* $, standard MP relies on the fact that for well-conditioned $X$, the
inner product $h(j)=\left\langle r,X_{j}\right\rangle $ is close
to $\beta_{j}$, and thus a large value of $h(j)$ indicates a nonzero
$\beta_{j}$. When outliers are present, MP fails because the $h(j)$'s
may be distorted significantly by maliciously corrupted $x_{i}$'s and
$y_{i}$'s. To protect against outliers, it is crucial
to obtain a robust estimate of $h(j).$ This motivates our robust
version of MP.

The proposed Robust Matching Pursuit algorithm (RoMP) is summarized in Algorithms \ref{alg:Robust-OMP} and \ref{alg:trim_inner_product}.
Similar to standard MP, it selects the columns of
$X$ with highest inner products with the residual. There are two main differences
from standard MP. The key difference is that we compute a robust version of inner product by trimming large points. Also, there is no iterative procedure -- we take the inner products between all the columns of $X$ and the response vector $y$ and selects the top $k$ ones; this leads to a simpler analysis. 

\begin{algorithm}
\caption{\label{alg:Robust-OMP}Robust Matching Pursuit (RoMP)}

Input:~$X,y,k,n_{1}$.

For~$j=1,\ldots,p$,~compute~the~trimmed~inner~product~(see~Algorithm~\ref{alg:trim_inner_product}):
\[
h(j)=\textrm{trimmed-inner-product}(y,X_{j},n_1).
\]

Sort~$\{\left|h(j)\right|\}$~and~select~the~$k$~largest~ones.~

Let~$\hat{\Lambda}$~be~the~set~of~selected~indices.

Set~$\hat{\beta}_j = h(j)$~for~$j\in\hat{\Lambda}$ and $0$ otherwise.

Output:~$\hat{\beta}$
\end{algorithm}
\begin{algorithm}
\caption{\label{alg:trim_inner_product}Trimmed Inner Product}

Input:~$a \in \mathbb{R}^N$,~$b\in\mathbb{R}^{N}$,~$n_{1}$

Compute~$q_{i}=a_{i}b_{i}$,~$i=1,\ldots,N$.

Sort~$\{|q_{i}|\}$ and select the smallest $(N-n_1)$ ones.

Let $\Omega$ be the set of selected indices.

Output:~$h=\sum_{i \in \Omega}p_{i}$.

\end{algorithm}

The key idea behind RoMP is that it effectively reduces a high-dimensional
robust regression problem to a much easier low-dimensional (2-D) problem,
one that is induced by the sparse structure of $\beta^{*}$. Outlier
rejection is performed in this low-dimensional space, and the support
of $\beta^{*}$ is estimated along the way. Hence, this procedure fulfils
our previous intuition of simultaneous structure identification and
outlier rejection.

Our algorithm requires two parameters, $ n_1 $ and $ k $. We discuss how to choose these parameters after we present the performance guarantees in the next section.

\section{Performance Guarantees for RoMP}

We are interested in finding conditions for $(p,k,n,n_{1})$ under
which RoMP is guaranteed to recover $\beta^{*}$ with correct support and small error.
We consider the following sub-Gaussian design model. Recall that a random variable $ Z $ is sub-Gaussian with parameter $ \sigma $ if $ \mathbb{E}[\exp(tZ)]\le \exp(t^2 \sigma^2 /2) $ for all real $t$.
\begin{defn}
[Sub-Gaussian design] Suppose the entries of $X^{\mca}$ are i.i.d. zero-mean sub-Gaussian variables with parameter $\frac{1}{\sqrt{n}}$ and variance $\frac{1}{n}$, and the entries of the additive noise are i.i.d. zero-mean sub-Gaussian variables with parameter $\frac{\sigma_e}{\sqrt{n}}$ and with variance $ \frac{\sigma_e^2}{n} $.
\end{defn}
Note that this general model covers the case of Gaussian, symmetric Bernoulli, and any other distributions with bounded support.

\subsection{Guarantees for the Distributed Corruption Model}

The following theorem characterizes the performance of RoMP, and shows that it can recover the correct support even when the number of outliers scales with $n$. In particular, this shows RoMP can tolerate an $O(1/\sqrt{k})$ fraction of distributed outliers. Recall that with high probability means with probability at least $ 1-p^{-2} $. 

\begin{thm}
\label{thm:RoMP} Under the Sub-Gaussian design model and the distributed corruption model, the following hold with high probability.

(1) The output of RoMP satisfies the following $ \ell_2 $ error bound:
\begin{eqnarray*}
\left\Vert \hat{\beta}-\beta^{*}\right\Vert _{2} & \lesssim & \left\Vert \beta^{*}\right\Vert _{2}\sqrt{1+\frac{\sigma_{e}^{2}}{\left\Vert \beta^{*}\right\Vert _{2}^{2}}}\left(\sqrt{\frac{k\log p}{n}} +\frac{n_{1}\sqrt{k}\log p}{n}\right).
\end{eqnarray*}
(2) If the nonzero entries of $\beta^*$ satisfy $\vert \beta^*_j\vert^2 \ge \left(\Vert \beta^* \Vert_2^2 /n\right)\log p\left(1+\sigma_{e}^{2}/\left\Vert \beta^{*}\right\Vert _{2}^{2}\right)$ , then RoMP correctly identifies the nonzero entries of $\beta^{*}$ provided 
\begin{eqnarray*}
n & \gtrsim & k\log p\cdot\left(1+\sigma_{e}^{2}/\left\Vert \beta^{*}\right\Vert _{2}^{2}\right), \textrm{ and}\\
\frac{n_{1}}{n} & \lesssim & 1/\left({\sqrt{k\left(1+\sigma_{e}^{2}/\left\Vert \beta^{*}\right\Vert _{2}^{2}\right)}\log p}\right).
\end{eqnarray*}
\end{thm}
The proof of the theorem is given in Section \ref{sec:proofs}. A few remarks are in order.
\begin{enumerate}
\item We emphasize that {\bf knowledge of the exact number of outliers is not needed} -- $n_1$ can be any {\bf upper bound} of the number of outliers, because by definition the adversary can change $n_1$ entries in each column arbitrarily, and changing less than $n_1$ of them is of course allowed.  The theorem holds even if there are less than $n_1$ outliers. Of course, this would result in sub-optimal bounds in the estimation due to over-conservativeness. In practice, cross-validation could be quite useful here.

\item We wish to note that essentially all robust statistical procedures we are aware of have the same character noted above. This is true even for the simplest algorithms for robustly estimating the mean. If an upper bound is known on the fraction of corrupted points, one computes the analogous trimmed mean. Otherwise, one can simply compute the median, and the result will have controlled error (but will be suboptimal) as long as the number of corrupted points is less than 50\% -- something which, as in our case, and every case, is always impossible to know simply from the data.

\item In a similar spirit, the requirement of the knowledge of $ k $ can also be relaxed. For example, if we use some $ k'> k $ instead of $ k $, then under the theorem continues to hold in the sense that RoMP identifies a superset (with size $ k' $) of the support of $ \beta^* $, and the $ \ell_2 $ error bound holds with $ k $ replaced by $ k' $. Standard procedures of estimating the sparsity level (e.g. cross-validation) can also be applied in our setting. 

\item Also note that the term $\sigma_e^{2}/\left\Vert \beta^{*}\right\Vert _{2}^{2}$ has a natural interpretation of signal to noise ratio.
\end{enumerate}

\subsection{Guarantee for the Row Corruption Model}
It follows directly from Theorem \ref{thm:RoMP} that RoMP can handle $n_1/2$ corrupted rows under the same condition. In fact, a slightly stronger results holds: RoMP can handle $ n_1 $ corrupted rows. This is the content of the follow theorem, with its proof given in Section \ref{sec:proofs}.
\begin{cor}
\label{cor:RoMP_row}
Under the Sub-Gaussian design model and the row corruption model with at most $n_1$ corrupted rows, the conclusions of Theorem \ref{thm:RoMP} holds. 
\end{cor}
Therefore, in particular, our algorithm is orderwise stronger than the Brute Force algorithm, in terms of the number of outliers it can tolerate while still correctly identifying the support.

\section{Experiments}
\label{sec:expt}

In this section, we report some simulation results for the performance of RoMP (Algorithm \ref{alg:Robust-OMP}) on synthetic data. The performance is measured in terms of support recovery (the number of non-zero locations of $\beta^*$ that are correctly identified), and also relative $\ell_2$-error ($\Vert \hat{\beta}-\beta^*\Vert_2 / \Vert \beta^* \Vert_2$).  The authentic data are generated under the sub-Gaussian Design model using Gaussian distribution with $p=4000,n=1600,k=10$ and $\sigma_e=2$, with the non-zero elements of $\beta^*$ being randomly assigned to $\pm 1$. 

For comparison, we also apply standard Lasso and JP \cite{laska2009jp,xiaodongli}  to the same data. For these algorithms, we search for the values of the tradeoff parameters $ \lambda, \gamma $ that yield the smallest $ \ell_2$-errors, and then estimate the support using the location of the largest $ k $ entries of $ \hat{\beta} $. It is also interesting to ask whether a simple modification of JP would perform well. While not analyzed in any of the papers that discuss JP, we consider JP with two different pre-processing procedures, both of which aim to detect and correct the corrupted entries in $ X $ directly. The first one, dubbed JP-fill, finds the set $ E $ of the largest $ \frac{n_1}{n} $ portion of the entries of $ X $, and then scales them to have unit magnitude. The second one, dubbed JP-row, discards the $ n_1 $ rows of $ X $ that contain the most entries in $ E $.

The corrupted rows $(X^\mco, y^\mco)$ are generated by the following procedure:
\begin{quote}
Let  $$ 
\theta^* = \arg \min_{\theta\in \mathbb{R}^{p-k}: \Vert \theta \Vert_1 \le \Vert \beta^* \Vert_1} \Vert y^\mca - X^\mca_{(\Lambda^*)^\top} \theta \Vert_2.
$$ 
Set $ X^\mco_{\Lambda^*} = \frac{3}{\sqrt{n}} A $, where $ A $ is a random $ \pm 1 $ matrix of dimension $ n_1 \times k $, and $ y^\mco = X^\mco_{\Lambda^*} (-\beta^*) $. For $i=1,\ldots,n_1$, further set 
$$ 
X^\mco_{i,(\Lambda^*)^c} = \left(y^\mco_i/(  B_{i}^\top \theta^*)\right) \cdot B_{i}^\top,
$$ 
where $ B_i $ is a $ (n-k) $-vector with i.i.d. standard Gaussian entries. 
\end{quote}

The results are shown in Figures \ref{fig:2}. It can be observed that RoMP performs better than Lasso and JP for both metrics, especially when the number of outliers is large. The $ \ell_2 $-errors of Lasso and JP flatten out because it returns a near-zero solution. The pre-processing procedures do not significantly improve performance of JP, which highlights the difficulty of performing outlier detection in high dimensions.

\begin{figure}[ht]
\vskip 0.2in
\begin{center}
\begin{tabular}{cc}
\includegraphics[scale=.55]{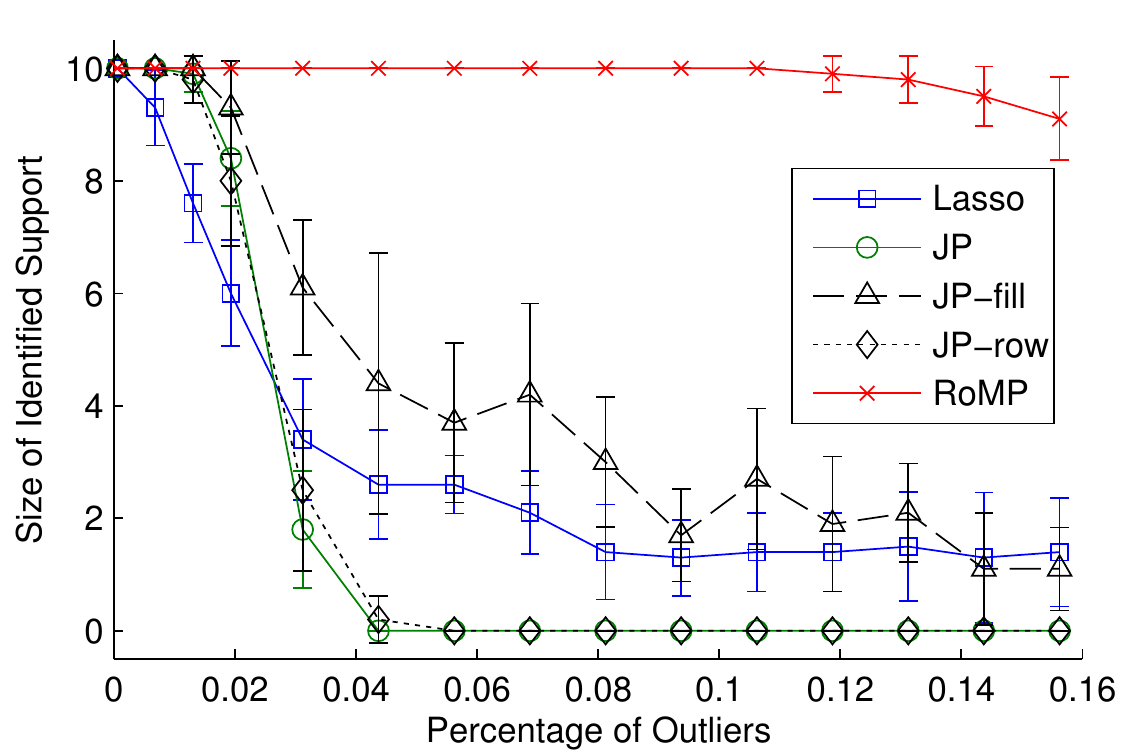} &
\includegraphics[scale=.55]{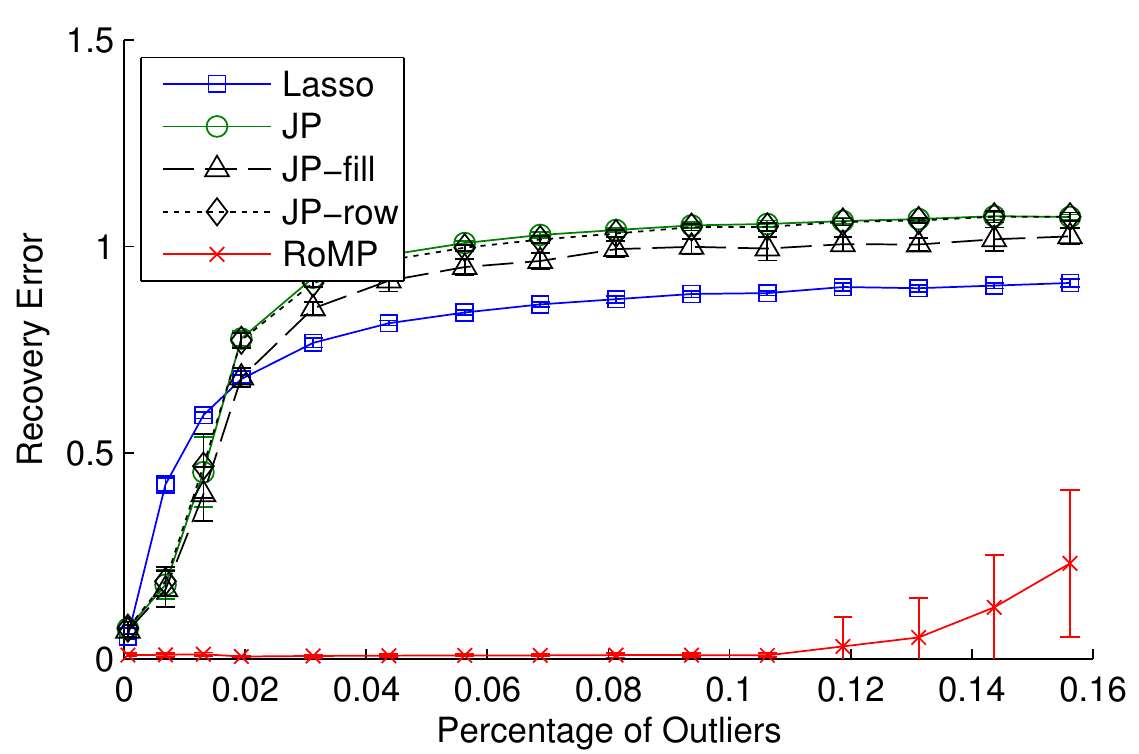} \\
(a) & (b)
\end{tabular}
\caption{Panel (a) shows the relative support recovery for different methods,  while Panel (b) shows the $\ell_2$ recovery errors. In both panels, the error bars correspond to one standard deviation. RoMP signficantly outperforms all other methods under both metrics. We note that the $ \ell_2 $ error of the other methods flattens out and error bars shrink to zero, because after about 4\% outliers, they return a near-zero solution.}
\label{fig:2}
\end{center}
\vskip -0.2in
\end{figure}

\section{Proofs}
\label{sec:proofs}

In this section, we turn to the proofs of Theorem \ref{thm:bruteforce}, Theorem \ref{thm:RoMP}, and Corollary \ref{cor:RoMP_row}, with some technical aspects of the proofs deferred to the appendix.

\subsection{Proof of Theorem 2}

For simplicity we assume $\Lambda^{*}=\{1,\ldots,k\}$, $\mathcal{A}=\{1,\ldots,n\}$,
and $\mco=\{n+1,\ldots,n+n_{1}\}$. We will show that the adversary
choose $y^{\mco}$ and $X^{\mco}$ in such a way that any ``correct''
solution of the form $(\theta,\mathcal{S},\Lambda^{*})$ (i.e., with
the correct support $\Lambda^{*})$ is not optimal because an
alternative solution $(\hat{\theta},\hat{\mathcal{S}},\hat{\Lambda})$
with $\hat{\theta}=[1,\ldots,1]^{\top}$, $\hat{\mcs}=\{n_{1}+1,\ldots,n+n_{1}\}$,
$\hat{\Lambda}=\{2,\ldots,k,k+1\}$ has smaller objective value.

Now for the details. The adversary chooses $\left(y^{\mco}\right)_{i}=\sqrt{k}$
for all $i$, $X_{\Lambda^{*}}^{\mco}=0$, and $X_{k+1}^{\mco}=y^{\mco}$,
hence $y^{\mco}-X_{\hat{\Lambda}}^{\mco}\hat{\theta}=0$. To compute
the objective values of the ``correct'' solution and the alternative
solution, we need a simple technical lemma, which follows from standard
results for the norms of random Gaussian matrix. The proof is given in the appendix.

\begin{lem}
\label{lem:tech_brute_force}If $n\gtrsim k^{3}\log p$, we have
\begin{eqnarray*}
\left\Vert e+X_{\Lambda^{*}}^{\mca}\delta\right\Vert _{2}^{2} & \ge & \left(1-\frac{1}{k}\right))\sigma_{e}^{2},\forall\delta\in\mathbb{R}^{k}\\
\left\Vert e^{\hat{\mathcal{S}}/\mco}+X_{1}^{\hat{\mathcal{S}}/\mco}-X_{k+1}^{\hat{\mathcal{S}}/\mco}\right\Vert _{2}^{2} & \le & \left(1+\frac{1}{k}\right)\left(1-\frac{n_{1}}{n}\right)(\sigma_{e}^{2}+2)
\end{eqnarray*}
 with high probability.
\end{lem}

Using the above lemma, we can upper-bound the objective value of the
alternative solution:
\begin{eqnarray}
\left\Vert y^{\hat{S}}-X_{\hat{\Lambda}}^{\mc{\hat{\mathcal{S}}}}\hat{\theta}\right\Vert _{2}^{2} & = & \left\Vert y^{\mco}-X_{\hat{\Lambda}}^{\mco}\hat{\theta}\right\Vert _{2}^{2}+\left\Vert y^{\hat{S}/\mco}-X_{\hat{\Lambda}}^{\mc{\hat{\mathcal{S}}}/\mco}\hat{\theta}\right\Vert _{2}^{2}\nonumber \\
 & = & 0+\left\Vert y^{\hat{S}/\mco}-X_{\Lambda^{*}}^{\mc{\hat{\mathcal{S}}}/\mco}\beta_{\Lambda^{*}}^{*}+X_{1}^{\hat{\mathcal{S}}/\mco}-X_{k+1}^{\hat{\mathcal{S}}/\mco}\right\Vert _{2}^{2}{}^{2}\nonumber \\
 & = & \left\Vert e_{\hat{\mathcal{S}}/\mco}+X_{1}^{\hat{\mathcal{S}}/\mco}-X_{k+1}^{\hat{\mathcal{S}}/\mco}\right\Vert _{2}^{2}\nonumber \\
 & \le & \left(1+\frac{1}{k}\right)\left(1-\frac{n_{1}}{n}\right)(\sigma_{e}^{2}+2).\label{eq:upper_bound_alt_soln}
\end{eqnarray}
To lower-bound the objective value of solutions of the form $(\theta,\mathcal{S},\Lambda^{*})$,
we distinguish two cases. If $\mathcal{S}\cap\mco\neq\phi$, then
the objective value is 
\begin{eqnarray}
\left\Vert y^{\mcs}-X_{\Lambda^{*}}^{\mcs}\theta\right\Vert _{2}^{2} & \ge & \left\Vert y^{\mcs\cap\mco}-X_{\Lambda^{*}}^{\mcs\cap\mco}\theta\right\Vert _{2}^{2}\nonumber \\
 & = & \left\Vert y^{\mcs\cap\mco}\right\Vert _{2}^{2}\nonumber \\
 & \ge & k\label{eq:lower_bound_good_soln1}
\end{eqnarray}
If $\mathcal{S}\cap\mco=\phi$, we have $\mcs=\mca$ and thus 
\begin{eqnarray}
\left\Vert y^{\mcs}-X_{\Lambda^{*}}^{\mcs}\theta\right\Vert _{2}^{2} & = & \left\Vert y^{\mca}-X_{\Lambda^{*}}^{\mca}\theta\right\Vert _{2}^{2}\nonumber \\
 &  & \left\Vert X_{\Lambda^{*}}^{\mca}\beta_{\Lambda^{*}}^{*}+e-X_{\Lambda^{*}}^{\mca}\theta\right\Vert _{2}^{2}\nonumber \\
 & = & \left\Vert e+X_{\Lambda^{*}}^{\mca}(\beta^{*}-\theta)\right\Vert _{2}^{2}\nonumber \\
 & \ge & \left(1-\frac{1}{k}\right)\sigma_{e}^{2}\label{eq:lower_bound_good_soln2}
\end{eqnarray}
where we use the lemma. When $n_{1}>\frac{3n}{k+1}$ and $\sigma_{e}^{2}=k$,
we have $\min\left\{ k,\left(1-\frac{1}{k}\right)\sigma_{e}^{2}\right\} >\left(1+\frac{1}{k}\right)\left(1-\frac{n_{1}}{n}\right)(\sigma_{e}^{2}+2)$.
Combining (\ref{eq:upper_bound_alt_soln}) (\ref{eq:lower_bound_good_soln1})
and (\ref{eq:lower_bound_good_soln2}) concludes the proof.

\subsection{Proof of Theorem 3}
We prove Theorem 3 in this section. We need two technical lemmas. The first lemma bounds the maximum of independent sub-Gaussian random variables. The proof follows from the definition of sub-Gaussianity and Chernoff bound, and is given in the appendix.
\begin{lem}
\label{lem:gaussian_max}
Suppose  $Z_{1},\ldots,Z_{m}$ are  $m$  independent sub-Gaussian random variables with parameter $\sigma$.
Then we have $\max_{i=1,\ldots,m}\left|Z_{i}\right|\le4\sigma\sqrt{\log m  + \log p}.$
with high probability.
\end{lem}
The second lemma is a standard concentration result for the sum of squares of independent sub-Gaussian random variables. It follows directly from Eq. (72) in \cite{loh2012nonconvex}.
\begin{lem}
\label{lem:subgaussian_concentr}
Let $Y_1, \ldots,Y_n$  be $n$ i.i.d. zero-mean sub-Gaussian random variables with parameter $\frac{1}{\sqrt{n}}$ and variance at most $\frac{1}{n}$. Then we have
\[
\vert\sum_{i=1}^{n} Y_i^2 - 1\vert \le c_1 \sqrt{\frac{\log p}{n}}  
\]
with high probability for some absolute constant $c_1$. Moreover, if $Z_1,\ldots,Z_n$ are also i.i.d. zero-mean sub-Gaussian random variables with parameter $\frac{1}{\sqrt{n}}$ and variance at most $\frac{1}{n}$, and independent of $Y_1,\ldots,Y_n$, then
\[
\vert\sum_{i=1}^{n} Y_i Z_i \vert \le c_2 \sqrt{\frac{\log p}{n}} 
\]
with high probability for some absolute constant $c_2$.
\end{lem}
\begin{rem}
When the above inequality holds, we write $\sum_{i=1}^{n}Y_{i}^2\approx 1\pm\sqrt{\frac{\log p}{n}}$ and $ \sum_{i=1}^n Y_i Z_i \approx \pm \sqrt{\frac{\log p}{n}} $.
w.h.p.
\end{rem}

Now consider the trimmed inner product $h(j)$ between the $j$th column
of $X$ and $y$. Let $\mca_j$ is the set of index $i$ such that $X_{ij}$ and $y_i$ are both not corrupted. By assumption $|\mca_j|\ge n$.  By putting $ |\mca_j|-n $ clean indices in $\mca_j^c$, we may assume $|\mca_j|=n$ without loss of generality. By prescription of Algorithm 2, we can write $h(j)$ as
\begin{eqnarray*}
h(j) =
  \sum_{i\in\mca_j}X_{ij}y_{i}-\sum_{\substack{i\in\textrm{trimmed}\\\textrm{inliers}}}X_{ij}y_{i}  +\sum_{\substack{i\in\textrm{remaining}\\\textrm{outliers}}} X_{ij}y_{i}.
\end{eqnarray*}
We estimate each term in the above sum.
\begin{enumerate}
\item Observe that
\begin{eqnarray*}
 \sum_{i\in\mca_j}X_{ij}y_{i}  =  \sum_{i\in \mca_j }X_{ij} \left(\sum_{k=1}^{p}X_{ik}\beta_{k}^{*}+e\right) 
 =  \sum_{i\in \mca_j }X_{ij}^2 \beta_{j}^{*} + \sum_{i\in \mca_j }X_{ij} \left(\sum_{k\neq j}X_{ik}\beta_{k}^{*}+e\right).
\end{eqnarray*}
(a) Because the points in $\mca_j$ obeys the Sub-Gaussian model, Lemma \ref{lem:subgaussian_concentr} gives $\sum_{i\in \mca_j }X_{ij}^2 \beta_{j}^{*} \approx \beta^*_j \left(1\pm\sqrt{\frac{1}{n}\log p}\right)$ w.h.p.

(b) On the other hand, because $X_{ik}$ and $X_{ij}$ are
independent when $k\neq j$, and $Z_i \triangleq \sum_{k\neq j}X_{ik}\beta_{k}^{*}+e$ are i.i.d. sub-Gaussian with parameter and standard deviation at most $\sqrt{\left(\left\Vert \beta^{*}\right\Vert _{2}^{2}+\sigma_{e}^{2}\right)/n}$, we apply Lemma \ref{lem:subgaussian_concentr} to obtain $  \sum_{i\in \mca_j }X_{ij} Z_i \approx\pm\frac{1}{\sqrt{n}}\sqrt{\left(\left\Vert \beta^{*}\right\Vert _{2}^{2}+\sigma_{e}^{2}\right)\log p}$ w.h.p.

\item Again due to independence and sub-Gaussianity of points in $\mca_j$, Lemma \ref{lem:gaussian_max} gives $\max_{i\in\mca_j}|X_{ij}|\lesssim \sqrt{(\log p) / n} $ w.h.p. and $\max_{i\in\mca_j}|y_i|\lesssim \sqrt{ (\log p /n )\left(\left\Vert \beta^{*}\right\Vert _{2}^{2}+\sigma_{e}^{2}\right)}$ w.h.p. It follows that w.h.p.
\begin{eqnarray*}
 \left|\sum_{\substack{i\in\textrm{trimmed}\\\textrm{inliers}}}X_{ij}y_{i}\right| 
 \le n_1 \left(\max_{i\in\mca} |X_{ij}|\right) \left( \max_{i \in \mca} |y_i| \right) 
 \lesssim  n_{1}\cdot\sqrt{\frac{\log p}{n}}\cdot\sqrt{\frac{\log p}{n}\left(\left\Vert \beta^{*}\right\Vert _{2}^{2}+\sigma_{e}^{2}\right)}.
\end{eqnarray*}

\item By prescription of the trimming procedure, either all outliers are trimmed, or the remaining outliers are no larger than the trimmed inliers. It follows from the last equation that w.h.p.
\begin{eqnarray*}
\left|\sum_{\substack{i\in\textrm{remaining}\\\textrm{outliers}}}X_{ij}y_{i} \right|
 \le \sum_{\substack{i\in\textrm{remaining}\\\textrm{outliers}}} |X_{ij} y_i|  \le \sum_{\substack{i\in\textrm{trimmed}\\\textrm{inliers}}} |X_{ij} y_i| 
 \lesssim  n_{1}{\frac{\log p}{n}}\cdot\sqrt{\left\Vert \beta^{*}\right\Vert _{2}^{2}+\sigma_{e}^{2}}.
\end{eqnarray*}
\end{enumerate}
Combining pieces, we have for all $j=1,\ldots,p$,
\begin{eqnarray}
\label{eq:x}
  \left|h(j)-\beta_{j}^{*}\right| 
 \lesssim  \left|\beta_{j}^{*}\right|\sqrt{\frac{2}{n}\log p}+\frac{1}{\sqrt{n}}\sqrt{\left(\left\Vert \beta^{*}\right\Vert _{2}^{2}+\sigma_{e}^{2}\right)\log p} 
   +n_{1}\cdot\frac{\log p}{n}\sqrt{\left(\left\Vert \beta^{*}\right\Vert _{2}^{2}+\sigma_{e}^{2}\right)}.
\end{eqnarray}

If RoMP correctly picks an index $j$ in the true support $\Lambda^{*}$,
then the error in estimating $\hat{\beta_{j}}$ is bounded by
the expression above. If RoMP picks some incorrect index $j$ not
in  $\Lambda^{*}$, then the difference between the
corresponding $\hat{\beta}_{j}$ and the true $\beta_{j'}^{*}$ that
should have been picked is still bounded by the expression above (up
to constant factors). Therefore, we have
\begin{eqnarray*}
   \left\Vert \hat{\beta}-\beta^{*}\right\Vert _{2}^{2} 
  \lesssim  \sum_{j\in{\Lambda^*}}\left[\left|\beta_{j}^{*}\right|\sqrt{\frac{2}{n}\log p}+\frac{1}{\sqrt{n}}\sqrt{\left(\left\Vert \beta^{*}\right\Vert _{2}^{2}+\sigma_{e}^{2}\right)\log p} \right. 
   \left.  +n_{1}\frac{\log p}{n}\sqrt{\left(\left\Vert \beta^{*}\right\Vert _{2}^{2}+\sigma_{e}^{2}\right)}\right]^{2}.
\end{eqnarray*}
The first part of the theorem then follows after straightforward algebra manipulation.
On the other hand, RoMP picks the correct support as long as $|h(j)| > |h(j')|$ for all $j\in \Lambda^*,j'\in (\Lambda^*)^c$. In view of Eq.\eqref{eq:x}, we require
\begin{eqnarray*}
n & \gtrsim & \max_j \left(\frac{\Vert\beta^*\Vert_2^2}{ \beta_j^2}\right)\cdot\log p\cdot\left(1+\sigma_{e}^{2}/\left\Vert \beta^{*}\right\Vert _{2}^{2}\right)\\
\frac{n_{1}}{n} & \lesssim & \frac{1}{\sqrt{\max_j \left(\frac{\Vert\beta^*\Vert_2^2}{ \beta_j^2}\right) \cdot \left(1+\sigma_{e}^{2}/\left\Vert \beta^{*}\right\Vert _{2}^{2}\right)}\log p}
\end{eqnarray*}
One verifies that the above inequalities are satisfied under the conditions in the second part of the theorem.

\subsection{Proof of Corollary 1}
A careful examination of the proof of Theorem 2 in the last section shows that, when there are $n_1$ corrupted rows, the set $\mca_j$ still has cardinality at least $n$, and the proof thus holds under the row corruption model.

\section{Conclusion}
Adversarial corruption seems to be significantly more difficult than corruption independent from the original data, and moreover, corruption in $X$ as well as $y$ appears more challenging than corruption only in $y$. To the best of our knowledge, no prior existing algorithms have provable performance in this setting, or in the more difficult yet setting of distributed corruption. This paper provides the first results for both these settings. Our results outperform Justice Pursuit, as well as the exponential time Brute Force algorithm; more generally we show that no convex optimization based approach improve on the results we provide. Generalizing our results to obtain a sequential OMP-like algorithm, and on the other side, understanding converse results, are important next steps.

\bibliographystyle{plain}
\bibliography{cydong1}

\appendix
\section{Proof of the Lemma \ref{lem:tech_brute_force}}

Let $\theta'=\left[\sigma_{e}\vert\delta^{\top}\right]^{\top}$. We
can write $\left\Vert e+X_{\Lambda^{*}}^{\mca}\delta\right\Vert _{2}^{2}=\left\Vert Z_{1}\theta'\right\Vert _{2}^{2}$
with $Z_{1}\triangleq\left[\frac{1}{\sigma_{e}}e\vert X_{\Lambda^{*}}^{\mca}\right]$.
Note that $Z_{1}$ is an $n\times(k+1)$ matrix with i.i.d. $\mathcal{N}(0,\frac{1}{n})$
entries, whose smallest singular value can be bounded using standard
results. For example, using Lemma 5.1 in \cite{baraniuk2008simpleRIP}with
$\Phi(\omega)=Z_{1}$, $N=k+1$, $T=\{1,\ldots,N\}$, $\delta=\frac{1}{3k}$
and $c_{0}(\delta/2)=1/288k^{2}$, we have 
\[
\left\Vert Z_{1}\theta'\right\Vert _{2}^{2}\le\left(1+\frac{1}{3k}\right){}^{2}\left\Vert \theta'\right\Vert _{2}^{2}\le\left(1+\frac{1}{k}\right)\sigma_{e}^{2},\forall\delta
\]
with probability at least 
\[
1-2e^{\frac{1}{288k^{2}}n-(k+1)\ln(36k)}\ge1-2p^{-3}
\]
provided $n\ge576(k+1)^{3}\ln(36p)$. This proves the first inequality.

\section{Proof of Lemma \ref{lem:gaussian_max}}

Let $\hat{Z}=\max_{i}Z_{i}$. By definition of sub-Gaussianity, we have
\begin{eqnarray*}
\mathbb{E}\left[e^{t \hat{Z}/\sigma}\right] & = & \mathbb{E}\left[\max_{i}e^{tZ_{i}/\sigma}\right]\\
 & \le & \sum_{i}\mathbb{E}\left[e^{tZ_{i}/\sigma}\right]\\
 & \le & m e^{t^{2}/2}\\
 & = & e^{t^{2}/2+\log m}
\end{eqnarray*}
It follows from Markov Inequality that
\begin{eqnarray*}
P(\hat{Z}\ge\sigma t) & = & P(e^{t\hat{Z}/\sigma}\ge e^{t^{2}})\\
 & \le & e^{-t^{2}}\mathbb{E}\left[e^{t\hat{Z}/\sigma}\right] \\
 & \le & e^{-t^{2}+t^{2}/2+\log m} \\
 & = &e^{-\frac{1}{2}t^{2}+\log m}.
\end{eqnarray*}

By symmetry we have 
\[
P(\min_{i}Z_{i}\le-\sigma t)\le e^{-\frac{1}{2}t^{2}+\log m},
\]
so a union bound gives
\begin{eqnarray*}
P(\max_{i}\left|Z_{i}\right|\ge\sigma t) & \le & P(\max_{i}Z_{i}\ge\sigma t)+P(\min_{i}Z_{i}\le-\sigma t)\\
 & \le & 2e^{-\frac{1}{2}t^{2}+\log m}.
\end{eqnarray*}
Taking $t=4\sqrt{\log m + \log p}$ yields the result.

\end{document}